\documentclass[10pt, conference, lettersize]{IEEEtran}

\usepackage{cite}
\usepackage{amsmath,amssymb,amsfonts}
\usepackage{algorithm}
\usepackage{algorithmic}
\usepackage{graphicx}
\usepackage{textcomp}
\usepackage{xcolor}
\usepackage{float}
\usepackage{subfigure}
\usepackage{subeqnarray}
\usepackage{adjustbox}
\usepackage{flushend}
\PassOptionsToPackage{hyphens}{url}\usepackage{hyperref}
\usepackage[justification=centering]{caption}

\usepackage{graphicx}  

\usepackage{caption}
 
\DeclareCaptionLabelFormat{lc}{{#1}~#2}
\usepackage{pifont} 

\makeatletter
\def\endthebibliography{%
  \def\@noitemerr{\@latex@warning{Empty `thebibliography' environment}}%
  \endlist
}
\makeatother

 \usepackage{epsfig,endnotes}

\usepackage{floatpag,enumitem}

\usepackage{relsize}
 
 \usepackage{tikz}

\usepackage{multirow}
\usepackage{setspace}
\usepackage{mathrsfs}
\usepackage{bbm}
\usepackage{pifont}
\usepackage{fancyhdr}
\usepackage{amsfonts}

\usepackage{amsmath, amsthm}

\usepackage{tabularx}
\usepackage{listings}
\usepackage{graphicx}
\usepackage{amssymb,booktabs}
\usepackage{array}
\usepackage{cases}

 \usepackage{supertabular}
\usepackage{mathrsfs}

\usepackage{psfrag}
\usepackage{bbding}

\usepackage{float}

\usepackage{amsbsy}

\makeatletter
\providecommand{\leftsquigarrow}{%
  \mathrel{\mathpalette\reflect@squig\relax}%
}
\newcommand{\reflect@squig}[2]{%
  \reflectbox{$\m@th#1\rightsquigarrow$}%
}
\makeatother

\usepackage{algorithm}
 \usepackage{algorithmic}

\makeatletter
\newcommand{\newalgname}[1]{%
  \renewcommand{\ALG@name}{#1}%
}

\newcommand {\C} {{\rm I\kern-5.5pt C}}

\renewcommand{\arraystretch}{1}

\def\centerhack#1{\hbox to 0pt{\hss\footnotesize #1\hss}}
\def\centerhackn#1{\hbox to 0pt{\hss #1\hss}}
\def\dchack#1{\vbox to 0pt{\vss{\hbox to 0pt{\hss#1\hss}}\vss}}

\usepackage{etoolbox}
\AtBeginEnvironment{align}{\setcounter{subeqn}{0}}
\newcounter{subeqn} %

\usetikzlibrary{positioning}

\newcounter{mysub}

\newtheorem{lem}{Lemma}
\newtheorem{assumption}{Assumption}

\newtheorem*{proposition1.1}{Proposition 1.1}
\newtheorem*{proposition1.2}{Proposition 1.2}
\newtheorem*{proposition1.3}{Proposition 1.3}
\newtheorem*{proposition2.1}{Proposition 2.1}
\newtheorem*{proposition2.2}{Proposition 2.2}
\usepackage{subfigure}

\begin{document}

%

\title{Time Minimization in Hierarchical Federated Learning}


%

%
%
%


\author{\IEEEauthorblockN{Chang Liu}
\IEEEauthorblockA{
\textit{Graduate College}\\
\textit{Nanyang Technological University}\\
Singapore\\
liuc0063@e.ntu.edu.sg}
\and
\IEEEauthorblockN{Terence Jie Chua}
\IEEEauthorblockA{
\textit{Graduate College}\\
\textit{Nanyang Technological University}\\
Singapore\\
terencej001@e.ntu.edu.sg}
\and
\IEEEauthorblockN{Jun Zhao}
\IEEEauthorblockA{
\textit{School of Computer Science \& Engineering}\\
\textit{Nanyang Technological University}\\
Singapore\\
junzhao@ntu.edu.sg}
}

\markboth{}%
{}
%



\maketitle
\thispagestyle{fancy}
\pagestyle{fancy}
\lhead{This paper appears in the Proceedings of 2022 ACM/IEEE Symposium on Edge Computing (SEC).\\ Please feel free to contact us for questions or remarks.} 

\cfoot{~\\[-25pt]\thepage}

\begin{abstract}
Federated Learning is a modern decentralized machine learning technique where user equipments perform machine learning tasks locally and then upload the model parameters to a central server. 
In this paper, we consider a 3-layer hierarchical federated learning system which involves model parameter exchanges between the cloud and edge servers, and the edge servers and user equipment. In a hierarchical federated learning model, delay in communication and computation of model parameters has a great impact on achieving a predefined global model accuracy. 
Therefore, we formulate a joint learning and communication optimization problem to minimize total model parameter communication and computation delay, by optimizing local iteration counts and edge iteration counts. To solve the problem, an iterative algorithm is proposed. 
After that, a time-minimized UE-to-edge association algorithm is presented where the maximum latency of the system is reduced. Simulation results show that the global model converges faster under optimal edge server and local iteration counts. The hierarchical federated learning latency is minimized with the proposed UE-to-edge association strategy.
\end{abstract}


\section{Introduction}
High-tech mobile devices and Internet of Things (IoT) are generating a large amount of data~\cite{bonawitz2019towards}. These immense volumes of data have incentivized high-speed development in big data technology and Artificial Intelligence. Conventional Machine Learning (ML) and Deep Learning (DL) methods require devices to upload their data to a central server to develop a global model. However, the threats involving leakages of, and attacks on privacy-sensitive data demotivate users to upload data from their user equipments (UE) to a central server for computing. 
Fortunately, rapid development in computing technology has spurred the age of Mobile Edge Computing (MEC), in which the computing power of chips in mobile devices is strengthening, facilitating more computing-intensive tasks such as machine learning tasks. Computing processes that were traditionally computed centrally at a server are shifting to mobile edge devices. Decentralized ML, which takes into account privacy concerns, has been coined Federated Learning (FL)~\cite{haddadpour2019convergence,zhao2018federated} and this model training technique involves user equipments or user-edge devices to perform ML tasks locally, after which the locally trained model parameters will be uploaded to a central server for global model parameter aggregation. Globally aggregated models are then downloaded by the UE, and that concludes a single round of FL. The process mentioned above is repeated up until a stopping criterion is met. The FL process facilitates devices to build a shared model while preserving the data privacy of the users.

With the rapid development of Federated Learning, federated learning also faces challenges over wireless networks. 
It is common for FL parameter transmissions to be undertaken by numerous participating user equipments over resource-limited networks, for example, wireless networks where the bandwidth or power is limited.
The repeated FL model parameter transmission between user equipments and servers, therefore, can cause a significant delay that can be as much or more than the machine learning model training time, which impairs the performance of latency-sensitive applications. Some of these works acknowledged the physical properties of wireless communication and proposed solutions such as analog model aggregation over the air (over-the-air computation). These analog aggregation techniques aim to reduce communication latency by allowing devices to upload models simultaneously over a multi-access channel~\cite{amiri2020machine,ahn2019wireless,chen2019artificial,tran2019federated}. Nevertheless, analog model aggregation over the air would require stringent synchronization conditions to be met.
Another genre of papers aims to minimize the overall FL convergence time (delay in both communication and local device computation) by allocating resources and solving optimization problems. Papers such as~\cite{chen2020convergence,shi2020device} established optimization problems which optimize variables such as UE uplink bandwidth and devices selection with the aim of minimizing FL convergence time.
Some other works proposed optimization problems with the aim of reducing FL training loss. In~\cite{chen2020joint}, the authors proposed a joint resource allocation and device scheduling optimization problem with the aim to minimize training loss. 
Similarly, the authors of~\cite{wang2019adaptive} proposed an optimization problem with the aim of minimizing training loss by a varying number of local update steps between global iterations, and the number of global aggregation rounds. 
However, these works focused on energy consumption and machine learning model performance but failed to consider delay minimization. \cite{yang2020delay} proposed a joint transmission and computation optimization problem aiming to minimize the total delay in FL, but it only considered the traditional 2-layer FL framework. In~\cite{luo2020hfel}, the edge server aggregation and edge model transmission delay are taken into consideration, but they are not incorporated into the objective function. 

\textbf{The challenges of minimizing the total delay in hierarchical federated learning framework.} 
Different from traditional 2-layer FL, in hierarchical FL the total time is determined not only by end devices, but also by edge servers. 
To achieve a specific global accuracy, more local iterations provide a more accurate local model while taking more local computation time, but the cost can be mitigated by reducing the number of edge aggregations. More edge aggregations may reduce the demand for local computing, but it incurs more communication delay. 
To tackle these problems, in this paper, we formulate a joint communication and learning optimization problem in order to find the optimal solutions for local training iterations and edge aggregation iterations. 
We give an analysis of the property of the proposed optimization problem. 

\textbf{Our work is novel and contributes to the existing literature by:} 

\begin{itemize}
\item Proposing a 3-layer hierarchical FL time minimization model, and setting up an optimization problem which aims to minimize the hierarchical FL time by optimizing the number of local UE computations, and the number of local aggregations under given global accuracy. 

\item Considering the edge server model aggregation and edge to cloud server transmission delay into our optimization objective function. To the best of our knowledge, there are no existing papers that utilize the edge server model aggregation and edge-to-cloud server transmission delay in the optimization objective function. 

\item Presenting a UE-to-edge association strategy that aims to minimize the system's latency. The results show that the proposed UE-to-edge association strategy achieves the minimum latency compared with other methods.

\end{itemize}

The rest of the paper is structured as follows. Section~\ref{section:Related} surveys related work. In Section~\ref{section: system model}, we describe the system model and give a framework for our hierarchical federated learning system. 
The problem is then formulated in section~\ref{section:hierarchical}, followed by the analysis and the optimal solutions of the optimization problem. In Section~\ref{section:numerical}, the numerical results are shown and analyzed. Finally, we give a conclusion in Section~\ref{section:conclusion}.

\section{Related Work}\label{section:Related}
There have been a lot of efforts to improve and analyze the performance of federated learning.

\textbf{Three-layer hierarchical federated learning.}
Many studies considered 3-layer federated learning. In~\cite{luo2020hfel}, the authors introduced a hierarchical FL edge learning framework in contrast to the traditional 2-layer FL systems proposed by the other above-mentioned papers. 
We should note that there are other papers which propose hierarchical architectures for their FL training that have other aims apart from those mentioned above. In~\cite{wang2020local}, the authors studied hierarchical federated learning with stochastic gradient descent and conducted a thorough analysis to analyze its convergence behavior. 
The works in~\cite{liu2020client} considered a client-edge-cloud hierarchical federated learning system and proposed a novel HierFAVG algorithm which allows edge servers to perform partial model aggregation to enable better communication-computation trade-off while allowing the model to be trained quicker.
The authors in~\cite{lim2021dynamic} utilized a Stackelberg differential game to model the optimal bandwidth allocation and reward allocation strategies in hierarchical federated learning.
\cite{mhaisen2021optimal} utilized branch and bound-based and heuristic-based solutions to minimize the data distribution distance at the edge level.
For IoT heterogeneous systems,~\cite{abdellatif2022communication}proposed an optimized user assignment and resource allocation solution over hierarchical FL architecture.
It can be seen that hierarchical FL is a promising solution that allows for the adaptive and scalable implementation by making use of the resources available at the edge of the network.

\textbf{Delay minimization in federated learning.}
The convergence time of FL was studied in many works.~\cite{chen2020convergence} jointly considered user selection and resource allocation in cellular networks to reduce the FL convergence time.
FedTOE~\cite{wang2021quantized} executed a joint allocation of bandwidth and quantization bits to minimize the quantization errors under transmission delay constraint.
\cite{chen2021communication} proposed to reduce the FL convergence time by reducing the volume of the model parameters exchanged among devices.
\cite{yang2020delay} proposed a joint transmission and computation optimization problem aiming to minimize the total delay in FL. 
However, these papers mainly focus on the resource allocation under the constraint of delay or convergence time but they failed to consider how user equipments themselves can reduce computation and communication delay from the frequency of communication between them and the edge servers while maintaining the required machine learning accuracy. Besides, they only studied the traditional 2-layer federated learning framework. In contrast to other papers, ~\cite{luo2020hfel} utilized a 3-layer, hierarchical model (UE-edge-cloud) for the optimization problem which aimed to minimize the weighted combination of FL convergence time and UE energy consumption. 
While the authors of~\cite{luo2020hfel} did consider the edge server aggregation and edge-to-cloud transmission delay in their paper, they did not propose or incorporate these factors into their proposed optimization objective function. 

\textbf{Novelty of our work.}
In this paper, we propose to minimize transmission and computation delay between the cloud and edge servers, and edge servers and UEs in a hierarchical FL framework, by optimizing the number of local UE computations and the number of local aggregations.
Although the above-mentioned works considered the FL convergence time minimization, they did not take the number of local computations and the number of edge aggregations into account. In our work, under given accuracy, the proposed method is able to find the optimal number of local computations, the number of edge aggregations and the UE-to-edge association strategy, thus providing an optimal global setting for 3-layer hierarchical FL system.

\section{System Model} \label{section: system model}
We consider a hierarchical federated learning model consisting of a cloud server $\mathcal{S}$, a set $\mathcal{M}$ of $M$ edge servers and a set $\mathcal{N}$ of $N$ user equipments (UEs) as shown in Fig.~\ref{fig1}. Each UE $n$ owns a local data set $\mathcal{D}_n$ with size $D_n$.

\subsection{Three-Layer Federated Learning Process}
The hierarchical federated learning process between UEs, edge and cloud is shown as follows. The procedure contains five steps: local computation at UE, local model transmission, edge aggregation, edge model transmission and cloud aggregation. 

\vspace{+5pt}
\subsubsection{Local computation}
Let $f_n$ be the CPU frequency for computation of UE $n$, and $C_n$ be the number of CPU cycles required for UE $n$ to compute one sample data. $D_n$ is the size of local data set,
then the time required in each iteration for computation of UE $n$ is 
\begin{align}
   t_n^{cmp} =  \frac{C_n D_n}{f_n}.
\end{align}
Let $a$ be the number of local iterations for each UE to perform in a single round of communication with the corresponding edge server.
In order to achieve a local accuracy $\theta \in (0,1)$, the number of local iterations $a$ that each UE needs to run is
\begin{align}\label{localUE}
  a=\zeta \ln \frac{1}{\theta},
\end{align}
where $\zeta$ is a constant depending on the loss function~\cite{yang2020energy}.

\vspace{+5pt}
\subsubsection{UE-to-edge model transmission}
After $a$ local iterations, UE uploads their local federated learning model to an edge server. 
We introduce the indicator variable $\chi_{n,m}$ which represents the association between UE $n$ and edge server $m$.
$\chi_{n,m} = 1$ means that UE $n$ uploads its local federated learning model to edge server $m$. Otherwise, $\chi_{n,m} = 0$.
Each UE can be associated with only one edge server. The user-server association rule can be described as:
\begin{align}
    \left\{
        \begin{aligned}
            &\chi_{n,m} \in \{0,1\},\forall n\in \mathcal{N},\forall m\in \mathcal{M}, \\
            &\sum_m \chi_{n,m}= 1,\forall n\in \mathcal{N}.
        \end{aligned}
    \right. 
\end{align}
Let the set of UEs that choose to transmit their local federated learning model to edge server $m$ be $\mathcal{N}_m$. 
Without loss of generality, \textit{orthogonal frequency division multiple access} (OFDMA) communication technique is adopted in this paper. According to Shannon's formula, the achievable transmission rate of UE $n$ and edge server $m$ can be formulated as
\begin{align}
    r_{n,m} = B_n \log_2 (1+\frac{g_{n,m}p_n}{N_0}), \label{r_n,m}
\end{align}
where $B_n$ is the bandwidth allocated to UE $n$, $g_{n,m}$ is the channel gain between UE $n$ and edge server $m$, and $p_n$ is the transmission power of UE $n$, and $N_0$ is the noise power. In this paper, we assume the bandwidth is equally allocated to all the UEs associated with the edge server.
Note that the total bandwidth each edge server $m$ can allocate is $\mathcal{B}$, so we have $\sum_n \chi_{n,m} B_n \leq \mathcal{B}, \forall m \in \mathcal{M}$.
Let $d_n$ denote the size of local model parameters. In wireless communication, downlink rate is typically much higher than uplink rate. Thus, the time for model downloads is negligible compared with computing time and model upload time. Then, the time cost for transmission federated learning model from UE $n$ to edge server within one round can be given as 
\begin{align}
   t_{n \to m}^{com}=\sum_m \chi_{n,m} \frac{d_n}{r_{n,m}}. \label{t_n_to_m}
\end{align}

\begin{figure}[tb]
\centering\includegraphics[width=0.55\textwidth]{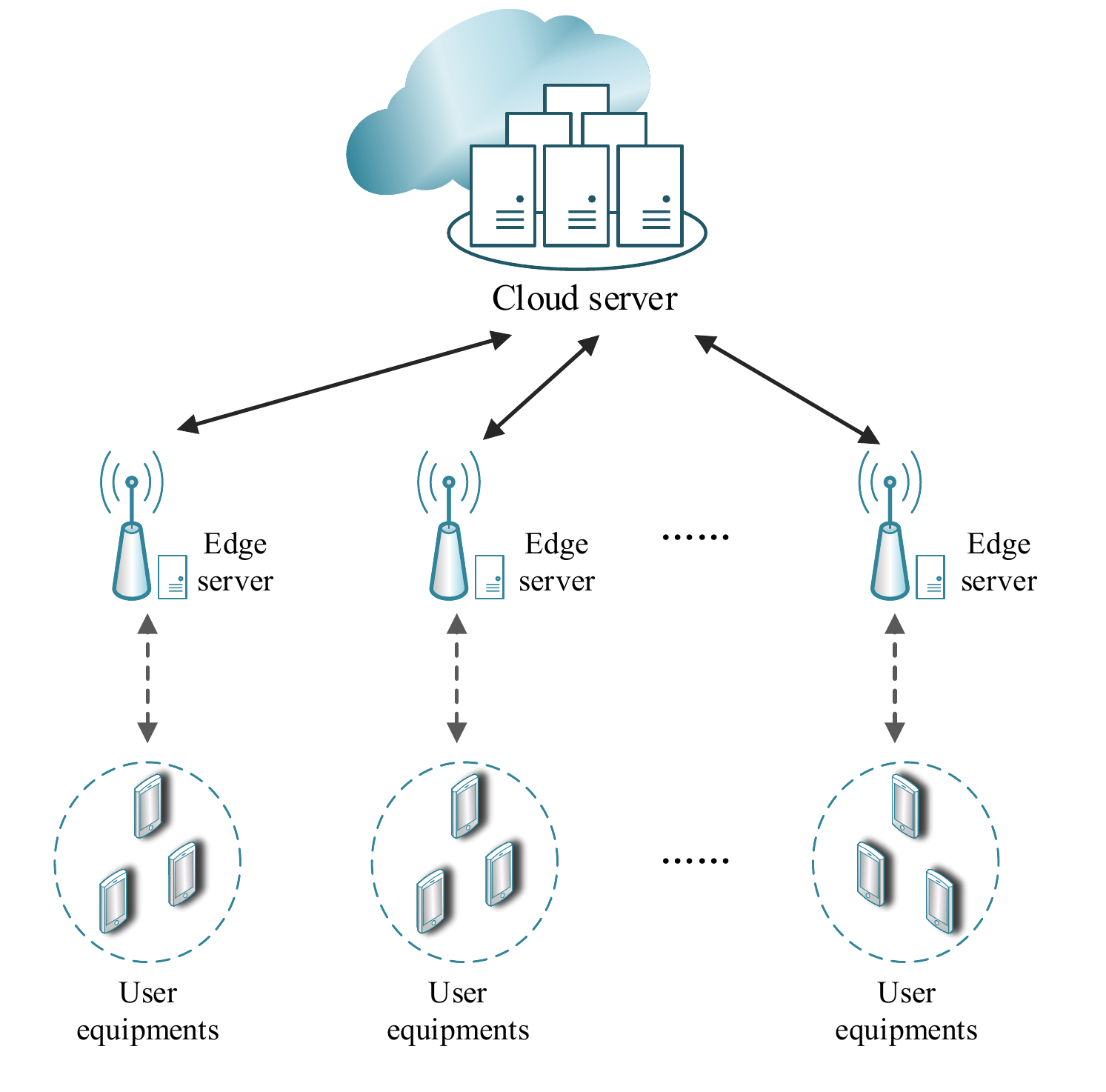}  
\caption{The architecture of the hierarchical federated learning model.}
\label{fig1}
\end{figure}

\subsubsection{Edge aggregation}
When edge server $m$ receives the model parameters transmitted from its associated UE $\mathcal{N}_m$, it obtains the averaged parameters $\boldsymbol{\omega}_m$ by 
\begin{align}
    \boldsymbol{\omega}_m = \frac{\sum_{n\in \mathcal{N}_m}D_n\omega_n}{D_{\mathcal{N}_m}},
\end{align}
where $D_{\mathcal{N}_m}:=\sum_{n \in \mathcal{N}_m} D_n$ is the total size of data aggregated at edge server $m$. Let $b$ be the number of iterations for each edge server to perform in a single round of communication with the cloud.
For simplicity, we use ``edge iterations" to represent the number of edge aggregations.
To achieve an edge accuracy $\mu$, for convex machine learning tasks, the number of edge iterations is given by~\cite{yang2020energy}
\begin{align}\label{edge}
    b=\frac{\gamma \ln (1/ \mu)}{1-\theta}.
\end{align}
From (\ref{edge}), it can be observed that $b$ is affected by both edge accuracy $\mu$ and local accuracy $\theta$. $\gamma$ is a constant related to the loss function and required loss function. It can be given by $\gamma=\frac{2L^2}{\beta^2 \delta}$~\cite{yang2020energy} where $\delta$ is a constant which is related to local training. When the model is required to be more accurate ($\mu$ and $\theta$ are small), the edge needs to run more iterations.
The delay  at edge server $m$ in each iteration is shown to be $\max_{n \in \mathcal{N}_m}\{a t_n^{cmp}+t_{n \to m}^{com}\}.$

\vspace{+15pt}
\subsubsection{Edge-to-cloud model transmission}
After $b$ rounds of edge aggregation, each edge server $m \in \mathcal{M}$ uploads its model parameters $\boldsymbol{\omega}_m$ to the cloud and downloads the global model from the cloud again. The delay within one round can be formulated as
\begin{align}
  t_{m \to c}^{com}=\frac{d_m}{r_m},
\end{align}
where $d_m$ is the size of model parameters at the edge server, and $r_m$ is the transmission rate between edge server $m$ and the cloud.

In order to achieve a global accuracy $\epsilon$, let the number of communications between edge server and cloud be $R(a,b,\epsilon)$. The global accuracy $\epsilon$ means that after $t$ iterations,
\begin{align}
    F(\boldsymbol{\omega}(t))-F(\boldsymbol{\omega}^*) \leq \epsilon\left[ F(\boldsymbol{\omega}(0))-F(\boldsymbol{\omega}^*)\right],
\end{align}
where $\boldsymbol{\omega}^*$ is the actual optimal global model.

\vspace{+10pt}
\subsubsection{Cloud aggregation}
Finally, the cloud aggregates the model parameters transmitted from the edge servers as follows:
\begin{align}
    \boldsymbol{\omega}=\frac{\sum_{m \in \mathcal{M}}D_{\mathcal{N}_m}\boldsymbol{\omega}_m}{D},
\end{align}
where $D:=\sum_{m \in \mathcal{M}}D_{\mathcal{N}_m}$ is the size of total data.

\vspace{+10pt}
\subsection{FL Model}
\vspace{+5pt}

In this work, we consider supervised federated learning. $\mathcal{D}_n=\{(\boldsymbol{x}_i,y_i)\}_{i=1}^{D_n}$ is the training set where $\boldsymbol{x}_i \in \mathbb{R}^d$ is the \mbox{$i$-th} input sample and $y_i \in \mathbb{R}$ is the corresponding label. Vector $\boldsymbol{\omega}$ is the parameters related to the FL model. We introduce the loss function $f(\boldsymbol{\omega},\boldsymbol{x}_i,y_i)$ that represents the loss function for one sample data. For different learning tasks, the loss function may be different. Loss function $F_n(\boldsymbol{\omega})$ on each UE $n$ is given by
\begin{align}
    F_n(\boldsymbol{\omega})=\frac{\sum_{i \in \mathcal{D}_n}f_i(\boldsymbol{\omega},\boldsymbol{x}_i,y_i)}{D_n}.
\end{align}

\newcommand{\norm}[1]{\left\lVert#1\right\rVert}
\begin{assumption} \label{Assumption}
	$F_n(\boldsymbol{\omega})$ is $\beta$-strongly convex and $L$-smooth. $\forall n$ and $\forall \boldsymbol{\omega}, \boldsymbol{\omega}' \in \mathbb{R}^d$:
	{\setlength{\abovedisplayskip}{-3pt}
    \setlength{\belowdisplayskip}{-1pt}
	\begin{align}
	    F_n(\boldsymbol{\omega}) &\geq F_n(\boldsymbol{\omega}') + \bigl\langle \nabla F_n (\boldsymbol{\omega}'), \boldsymbol{\omega} - \boldsymbol{\omega}' \bigr\rangle  + \frac{\beta }{2} \norm{\boldsymbol{\omega} - \boldsymbol{\omega}'}^2,  \nonumber \\
	    &\norm{\nabla F_n(\boldsymbol{\omega}') - \nabla F_n(\boldsymbol{\omega})} \leq L\norm{\boldsymbol{\omega}' - \boldsymbol{\omega}}. \nonumber 
	\end{align}
	}
\end{assumption}

The training process is to minimize the global loss function $F(\boldsymbol{\omega})$, which can be formulated as
\begin{align}
    F(\boldsymbol{\omega})=\frac{\sum_{i=1}^N D_iF_i(\boldsymbol{\omega})}{D}.
\end{align}
We utilize distributed approximate Newton algorithm (DANE) ~\cite{shamir2014communication} to train the FL model. DANE is one of the most popular communication-efficient distributed training algorithm which is designed to solve general optimization problems. At each iteration, DANE takes an inexact Newton step appropriate for the geometry of objective problem. 
Stochastic gradient descent (SGD) is widely utilized since the computational complexity of SGD is low. However, it can only be used where the requirement on accuracy is not strict. Besides, SGD requires more iterations than gradient descent (GD). 
In federated learning, the wireless communication resource is valuable, so we use GD in UE local training.
At each iteration, each UE n updates its local model parameters. 
According to (\ref{localUE}), the UE needs to run $a$ iterations to achieve a local accuracy $\theta$. 
After every $a$ local iterations, each UE uploads its local model to the corresponding edge server, and edge server aggregates these models. Then after every $b$ edge server aggregations, the edge server uploads its model to the cloud, and the cloud performs global aggregation. Table \ref{table_notations} summarizes the notations.

\begin{table*}[t]
\renewcommand\arraystretch{1.5}
\centering
\small
\caption{Meaning of Notations} \label{table_notations}
\begin{tabular}{|m{1.6cm}<{\centering}||m{6cm}<{\centering}|m{1.6cm}<{\centering}||m{6cm}<{\centering}|}\hline  
\textbf{Notation} & \textbf{Description} & \textbf{Notation} & \textbf{Description} \\\hline  
    $\mathcal{S}$   &cloud server&$\mathcal{M}$&set of edge servers  \\\hline  
    $\mathcal{N}$  & set of user equipments & $\mathcal{N}_m$ & set of UEs that communicate with edge server $m$ \\ \hline
    $n$ & a single user equipment & $m$ & a single edge server \\ \hline
    $\mathcal{D}_n$&UE $n$'s local data set &$C_n$ &number of CPU cycles for UE to compute one input data \\\hline
    $D_n$&size of local data set at UE $n$&$D_{\mathcal{N}_m}$&size of data under edge server $m$ with set of UE $\mathcal{N}_m$ \\\hline
       $f_n$&UE $n$'s CPU frequency variable & $d_n$  &size of local model parameters at UE $n$  \\\hline  
       $t_n^{cmp}$& time for computation each iteration at UE $n$& $t_{n \to m}^{com}$&time for transmission from UE $n$ to edge server $m$\\\hline 
       $r_{n,m}$&transmission rate of UE $n$ and edge server $m$ &$p_n$&transmission power of UE $n$\\\hline
       $B_{n}$&bandwidth that allocated to UE $n$ & $g_{n,m}$& channel gain between UE $n$ and edge server $m$ \\ \hline
       $\chi_{n,m}$ & indicator variable of UE-to-edge association &$\mathcal{B}$ &total bandwidth each edge server can allocate \\\hline
       $a$&number of local iterations for UE in a round&$b$&number of iterations for edge server in a round \\ \hline
       $R(a,b,\epsilon)$&number of communications between edge server and cloud&$t_{m \to c}^{com}$&time for transmission from edge server $m$ to cloud \\\hline
       $\theta$&local accuracy & $\mu$&edge accuracy \\\hline
       $\epsilon$&global accuracy&$N_0$&noise power \\\hline
\end{tabular}
\end{table*}

\begin{algorithm}[tb]
\caption{Hierarchical FL Algorithm} 
    \begin{algorithmic}[1]
        \STATE {Initialize all user equipments (UEs) with parameter $w_n(0)$} \vspace{-10pt}
        \STATE{$i \gets 1$}
        \WHILE{global accuracy $\epsilon$ is not obtained}
            \STATE{each UE computes $\nabla F_n(\boldsymbol{\omega}_n(i))$ and sends it to the edge server)}.
            \STATE{each edge server computes $\nabla F(\boldsymbol{\omega}_m(i))=\frac{1}{N}\sum_{n=1}^N\nabla F_n(\boldsymbol{\omega}_n(i))$ and broadcasts it to all the UEs}.
            \FOR{each UE $n=1,2,\dotsc,N$ in parallel}
            \STATE{update $\boldsymbol{\omega}_n(t)$}.
            \ENDFOR
            \IF{$i \mid a =0$}
            \FOR{each UE $l=1,2,\dotsc,N$ in parallel}
            \STATE{communicate with its corresponding edge server and edge server perform edge aggregation.}
            \ENDFOR
            \ENDIF
            \IF{$i \mid a b = 0$}
            \FOR{each edge server $k=1,2,\dotsc,M$ in parallel}
            \STATE{communicate with cloud and cloud perform aggregation.}
            \ENDFOR
            \ENDIF
            \STATE {$i \gets i+1$}
        \ENDWHILE
    \end{algorithmic}
\end{algorithm}

\section{Hierarchical federated learning delay optimization}

\label{section:hierarchical}
In this section, we give a formulation of the proposed problem followed by analysis and solutions. 

\subsection{Problem Formulation}
Given the system model above, we now formulate the hierarchical federated learning time optimization problem:
\begin{subequations}\label{p1}
\begin{align}
\min_{a,b,\boldsymbol{\chi},\boldsymbol{f},\boldsymbol{p}}~&R(a,b,\epsilon)\max_{m \in \mathcal{M}}\{b\max_{n \in \mathcal{N}_m}\{a t_n^{cmp} + t_{n \to m}^{com}\} + t_{m \to c}^{com}\},\tag{\ref{p1}}\\
\text{s.t. }
&0<f_n \leq f_n^{max},~\forall n \in \mathcal{N}_m,\label{p1_1}\\
&0<p_n \leq p_n^{max},~\forall n \in \mathcal{N}_m,\label{p1_2}\\
&\chi_{n,m} \in \{0,1\},\forall n\in \mathcal{N},\forall m\in \mathcal{M}, \label{p1_3}\\
&\sum_m \chi_{n,m} = 1,\forall n\in \mathcal{N},\label{p1_4}\\
& \sum_n \chi_{n,m} B_n \leq \mathcal{B}, \forall m \in \mathcal{M},\label{p1_5}\\
& a, b \in \mathbb{N}^+.\label{p1_6}
\end{align}
\end{subequations}
where $a t_n^{cmp} + t_{n \to m}^{com}$ is the time taken for UE $n$ to perform $a$ iterations of local computation and one round of communication with edge server $m$. Thus, $\max_{n \in \mathcal{N}_m}\{a t_n^{cmp} + t_{n \to m}^{com}\}$ is the time taken for one single round of communication between all UEs and its corresponding edge servers, and $\max_{m \in \mathcal{M}}\{b\max_{n \in \mathcal{N}_m}\{a t_n^{cmp} + t_{n \to m}^{com}\} + t_{m \to c}^{com}\}$ is the time taken for a single round of communication between all edge servers and the cloud. In the objective function, the total delay of the entire federated learning task is minimized.
Constraints (\ref{p1_1}) and (\ref{p1_2}) are the maximum CPU frequency and transmission power of UEs; constraints (\ref{p1_3}) and (\ref{p1_4}) are the UE-edge server association rules;  constraint (\ref{p1_5}) guarantees that the bandwidth of each edge server does not exceed the upper bound limit; constraint (\ref{p1_6}) specifies that $a$ and $b$ are integers.
This optimization problem falls into the category of integer programming since $a, b$ are positive integer values. While integer programming is considered to be an NP-hard problem~\cite{schrijver1998theory}, we could obtain sub-optimal solutions by relaxing the integer constraints and allowing $a, b$ to be continuous variables, which are rounded back to integer numbers later. According to (\ref{localUE}), $\theta$ can be expressed as $\theta=1/e^{\frac{a}{\zeta}}$. According to (\ref{edge}), $\mu$ can be expressed as $\mu=1/e^{\frac{b}{\gamma}(1-\theta)}$. The number of communications between edge server and cloud $R(a,b,\epsilon)$ is given by:
\begin{align}
    R(a,b,\epsilon)=\frac{C\ln(\frac{1}{\epsilon})}{1-\mu}.
\end{align}
Substitute $\mu$ with $\mu=1/e^{\frac{b}{\gamma}(1-\theta)}$ and $\theta$ with $\theta=1/e^{\frac{a}{\zeta}}$, we get:
\begin{align}
    R(a,b,\epsilon)=\frac{C\ln(\frac{1}{\epsilon})}{1-e^{-\frac{b}{\gamma}(1-e^{-\frac{a}{\zeta}})}}.
\end{align}
\subsection{Analysis}
In this section, we design an algorithm to solve the min-max problem (\ref{p1}). By introducing new slack variables $T$ and $\tau$, problem (\ref{p1}) is equivalent to the following optimization problem:
\begin{subequations} \label{p2}
\begin{align}
    \min_{\boldsymbol{a},\boldsymbol{b},\boldsymbol{\chi},T,\boldsymbol{\tau},\boldsymbol{f},\boldsymbol{p}}~&R(a,b,\epsilon) \cdot T, \tag{\ref{p2}} \\
    \text{s.t.~}
    & b\tau_m  + t_{m \to c}^{com} \leq T,~\forall m \in \mathcal{M}, \label{p2_1}\\
    & a t_n^{cmp} + t_{n \to m}^{com} \leq \tau_m,~\forall n \in \mathcal{N}_m, \label{p2_2}\\
    & 0<f_n \leq f_n^{max},~\forall n \in \mathcal{N}_m,\label{p2_3}\\
    & 0<p_n \leq p_n^{max},~\forall n \in \mathcal{N}_m,\label{p2_4}\\
    &\chi_{n,m} \in \{0,1\},\forall n\in \mathcal{N},\forall m\in \mathcal{M}, \label{p2_5}\\
&\sum_m \chi_{n,m} = 1,\forall n\in \mathcal{N},\label{p2_6}\\
& \sum_n \chi_{n,m} B_n \leq \mathcal{B}, \forall m \in \mathcal{M}.\label{p2_7}
\end{align}
\end{subequations}
where $T$ defines the time interval between each round of communication between edge server $m$ and the cloud, and $\tau_m$ defines the time interval between each round of communication between UE $n$ and edge server $m$. 
We should note that constraints (\ref{p2_1}) and (\ref{p2_2}) confine the delay (aggregation and communication to the cloud) at each edge server and the delay (computation and communication to the edge server) at each UE, respectively. 
To solve the optimization problem, we decompose the problem into two sub-problems. Sub-problem \uppercase\expandafter{\romannumeral1} solves the local iteration number $a$, edge iteration number $b$, UE CPU frequency $f$ and transmission power $p$. Sub-problem \uppercase\expandafter{\romannumeral2} obtains the optimal UE-to-edge association.

\subsection{Solution of sub-problem \uppercase\expandafter{\romannumeral1}}

We will show that Problem (\ref{p2}) is a convex optimization problem under given UE-to-edge association $\boldsymbol{\chi}$. To this end, we present lemmas below.

\begin{lem}\label{lem1}
The reciprocal of a positive and concave function is convex.
\end{lem}
\begin{proof}
Let $h(x)=\frac{1}{f(x)}$, where $f(x)$ is twice differentiable. Then the second order derivative of $h(x)$ can be given by:
\begin{align}
h''(x)=-\frac{f''(x)f(x)-2f'^2(x)}{f^3(x)}.
\end{align}
Since $f(x)$ is positive and concave, $f(x) > 0$ and $f''(x) < 0$. Thus $f''(x)f(x)-2f'^2(x)<0$, and $h''(x)>0$. As a result, $h(x)$ is a convex function.
\end{proof}

We have Lemma~\ref{lem2} below, which together with Lemma~\ref{lem1} shows that $R(a,b,\epsilon) \cdot T$ is convex.

\begin{lem}\label{lem2}
$\frac{1}{R(a,b,\epsilon) \cdot T}$ is a positive and concave function.
\end{lem}

\begin{proof}
The reciprocal of the objective of problem (\ref{p2}) is $$\frac{1}{R(a,b,\epsilon) \cdot T} = \frac{1-e^{-\frac{b}{\gamma}(1-e^{-\frac{a}{\zeta}})}}{CT\ln(\frac{1}{\epsilon})}.$$ Let $$f(a,b)=1-e^{-\frac{b}{\gamma}(1-e^{-\frac{a}{\zeta}})},$$ where $a,b,\zeta,\gamma$ are all positive numbers. Thus, $0<e^{-\frac{a}{\zeta}}<1$. Then $0<1-e^{-\frac{a}{\zeta}}<1$, and $0<1-e^{-\frac{b}{\gamma}(1-e^{-\frac{a}{\zeta}})}<1$. Besides, $C,T,\ln(\frac{1}{\epsilon} )$ are positive numbers. Therefore, $\frac{1}{R(a,b,\epsilon) \cdot T}>0$. 

Then, we investigate the concavity of $f(a,b)$. Let $g(x)=1-e^{-x}$, then $f(a,b)$ can be expressed as $$f(a,b)=g\bigg[\frac{b}{\gamma}\cdot g\bigg(\frac{a}{\zeta}\bigg)\bigg].$$ The second-order partial derivatives $f_{aa},f_{bb},f_{ab}$ can be given by:
\begin{align}
    f_{aa}&=\frac{b}{\gamma\zeta}\Bigg\{\frac{b}{\gamma\zeta}g''\left[\frac{b}{\gamma} g\left(\frac{a}{\zeta}\right)\right]g'^2\left(\frac{a}{\zeta}\right)\\\nonumber&~~~~~~~~~~~~~~~~~~+\frac{1}{\zeta}g'\left[\frac{b}{\gamma} g\left(\frac{a}{\zeta}\right)\right]g''\left(\frac{a}{\zeta}\right)\Bigg\},\\
    f_{bb}&=\frac{1}{\gamma^2}g''\left[\frac{b}{\gamma} g\left(\frac{a}{\zeta}\right)\right]g^2\left(\frac{a}{\zeta}\right),\\
    f_{ab}&=\frac{b}{\gamma^2\zeta}g''\left[\frac{b}{\gamma} g\left(\frac{a}{\zeta}\right)\right]g'\left(\frac{a}{\zeta}\right)g\left(\frac{a}{\zeta}\right)\\\nonumber&~~~~~~~~~~~~~~+\frac{1}{\gamma\zeta}g'\left[\frac{b}{\gamma} g\left(\frac{a}{\zeta}\right)\right]g'\left(\frac{a}{\zeta}\right).
\end{align}
Note that $g\left(x\right)=1-e^{-x}$, $g'\left(x\right)=e^{-x}$ and $g''\left(x\right)=-e^{-x}$. Thus, $g''\left(x\right)=-g'\left(x\right)$. Then, $f_{aa},f_{bb},f_{ab}$ can be rewritten as:
\begin{align}
    f_{aa}&=\frac{b}{\gamma \zeta^2}g'\left(\frac{a}{\zeta}\right)g'\left[\frac{b}{\gamma} g\left(\frac{a}{\zeta}\right)\right]\left[-\frac{b}{\gamma}g'\left(\frac{a}{\zeta}\right)-1\right],\label{f_aa}\\
    f_{bb}&=-\left[\frac{1}{\gamma}g\left(\frac{a}{\zeta}\right)\right]^2g'\left[\frac{b}{\gamma} g\left(\frac{a}{\zeta}\right)\right],\label{f_bb}\\
    f_{ab}&=\frac{1}{\gamma \zeta}g'\left(\frac{a}{\zeta}\right)g'\left[\frac{b}{\gamma} g\left(\frac{a}{\zeta}\right)\right]\left[-\frac{b}{\gamma}g\left(\frac{a}{\zeta}\right)+1\right].\label{f_ab}
\end{align}The Hessian matrix of $f(a,b)$ is
\begin{align}
    \left[ \begin{array}{cc}
f_{aa} & f_{ab} \\
f_{ba} & f_{bb} 
\end{array} 
\right ].
\end{align}
Since $f(a,b)$ is twice differentiable, the Hessian matrix is symmetric. That is, $f_{ab}=f_{ba}$. Next, we show that $f_{aa}<0$ and $f_{aa} \cdot f_{bb}-f^2_{ab} \geq0$.\par
From $g(x)=1-e^{-x}$ and $g'(x)=e^{-x}$, we have for any $x$, $g(x)>0$ and $g'(x)>0$. $b>0, \gamma >0$, so $-\frac{b}{\gamma}g'(\frac{a}{\zeta})-1<0$. Therefore, $f_{aa}<0$.

From (\ref{f_aa}), (\ref{f_bb}) and (\ref{f_ab}), it can be obtained that
\begin{align}
   & f_{aa} \cdot f_{bb}-f^2_{ab}=\frac{1}{\gamma^2\zeta^2}g'^2\left[\frac{b}{\gamma} g\left(\frac{a}{\zeta}\right)\right]g'\left(\frac{a}{\zeta}\right)\cdot \nonumber\\&\left\{\frac{b}{\gamma}g^2\left(\frac{a}{\zeta}\right)\left[\frac{b}{\gamma}g'\left(\frac{a}{\zeta}\right)+1\right]-g'\left(\frac{a}{\zeta}\right)\left[1-\frac{b}{\gamma}g\left(\frac{a}{\zeta}\right)\right]^2\right\}.
\end{align}
It is clear that $$\frac{1}{\gamma^2\zeta^2}g'^2\left[\frac{b}{\gamma} g\left(\frac{a}{\zeta}\right)\right]g'\left(\frac{a}{\zeta}\right) >0.$$ Next, we investigate the sign of $$\frac{b}{\gamma}g^2\left(\frac{a}{\zeta}\right)\left[\frac{b}{\gamma}g'\left(\frac{a}{\zeta}\right)+1\right]-g'\left(\frac{a}{\zeta}\right)\left[1-\frac{b}{\gamma}g\left(\frac{a}{\zeta}\right)\right]^2.$$ Let $g\big(\frac{a}{\zeta}\big)=t$ and $\frac{b}{\gamma}=k$, and then it can be expressed as
\begin{align}
    &kt^2(k-kt+1)-(1-t)(1-kt)^2 \\
    =&-kt^2+2kt+t-1 \\
    =&kt(2-t)-(1-t),
\end{align}
where $k>0$ and $t \in (0,1)$.Since $kt$ is a relatively large number, $kt(2-t) \geq (1-t)$. Therefore, $kt^2(k-kt+1)-(1-t)(1-kt)^2 \geq 0$. Thus,  $f_{aa} \cdot f_{bb}-f^2_{ab} \geq0$. Then Lemma \ref{lem2} holds.
\end{proof}
\begin{lem}\label{lem3}
Problem (\ref{p2}) under given UE-to-edge association $\boldsymbol{\chi}$ is a convex optimization problem.
\end{lem}
\begin{proof}

The constraints (\ref{p2_1}),~(\ref{p2_2}),~(\ref{p2_3}) and (\ref{p2_4}) are convex. Therefore, the convexity of problem (\ref{p2}) depends on the objective function of problem (\ref{p2}). From Lemmas~\ref{lem1} and~\ref{lem2}, we can conclude that the objective function of problem (\ref{p2}) is a convex function with respect to $a,b$.
\end{proof}

The Lagrange function of problem (\ref{p2}) can be given by:
\begin{align}
    \mathcal{L}(a,b,&T,\boldsymbol{f},\boldsymbol{p},\boldsymbol{\lambda},\boldsymbol{\mu},\boldsymbol{\beta},\boldsymbol{\nu})=R(a,b,\epsilon) \cdot T  \nonumber \\
    &+\sum_{m \in \mathcal{M}} \lambda_m(b\tau_m  + t_{m \to c}^{com}-T) \nonumber \\
    &+\sum_{n \in  \mathcal{N}_m}\mu_n(a t_n^{cmp} + t_{n \to m}^{com}-\tau_m)\nonumber \\
    &+\sum_{n \in  \mathcal{N}_m}\beta_n(f_n-f_n^{max}) \nonumber \\
    &+\sum_{n \in  \mathcal{N}_m}\nu_n(p_n-p_n^{max}),
\end{align}
where $\lambda_m$ and $\mu_n$ are Lagrangian multipliers associated with the constraints (\ref{p2_1}) and (\ref{p2_2}). Then the dual function of problem (\ref{p2}) is $$g(\boldsymbol{\lambda},\boldsymbol{\mu},\boldsymbol{\beta},\boldsymbol{\nu})= \min_{a,b,T,\boldsymbol{f},\boldsymbol{p}}\mathcal{L}(a,b,T,\boldsymbol{f},\boldsymbol{\lambda},\boldsymbol{\mu},\boldsymbol{\beta},\boldsymbol{\nu}).$$ 
According to the Karush-Khun-Tucker (KKT) conditions, the optimal solution of problem (\ref{p2}) can be obtained by taking the partial derivatives of Lagrange function $\mathcal{L}(a,b,\boldsymbol{f},\boldsymbol{p},\boldsymbol{\lambda},\boldsymbol{\mu})$ with respect to variable $a$ and $b$:
\begin{align}\label{partial derivatives}
    \frac{\partial \mathcal{L}}{\partial a}=&-\frac{CTb \ln (1/\epsilon) e^{-\frac{b}{\gamma}(1-e^{{-\frac{a}{\zeta}}}){-\frac{a}{\zeta}}}}{\gamma\zeta{(1-e^{-\frac{b}{\gamma}(1-e^{{-\frac{a}{\zeta}}})})}^2}\nonumber \\+&\sum_{n \in  \mathcal{N}_m}\mu_n t_n^{cmp}=0,\nonumber\\
    \frac{\partial \mathcal{L}}{\partial b}=&-\frac{CT \ln (1/\epsilon) e^{-\frac{b}{\gamma}(1-e^{{-\frac{a}{\zeta}}})}(1-e^{{-\frac{a}{\zeta}}})}{\gamma{(1-e^{-\frac{b}{\gamma}(1-e^{{-\frac{a}{\zeta}}})})}^2}\nonumber\\
    &+\sum_{m \in  \mathcal{M}}\lambda_m \tau_m =0.
\end{align}

\subsubsection{The optimal solution of local CPU frequency and transmission power} 
From constraints (\ref{p2_1}) and (\ref{p2_3}), it can be seen that it is always efficient to utilize the maximum CPU frequency $f_n^{max}, \forall n \in  \mathcal{N}_m$. Besides, from constraints (\ref{p2_2}) and (\ref{p2_4}), it can be seen that minimum time can be achieved if UE uses the maximum transmission power $p_n^{max}, \forall n \in  \mathcal{N}_m$. So the optimal solution of local CPU frequency and transmission power can be given by $f_n^*=f_n^{max}, p_n^*=p_n^{max}, \forall n \in  \mathcal{N}_m$.

\subsubsection{The optimal solution of local and edge iteration times ($a^*,b^*$) within one round communication}Let $A=CT \ln(1/\epsilon)$ and $Y=1-e^{-\frac{a}{\zeta}}$, we can get:
\begin{align}
    a^*=&\zeta \ln\bigg(\frac{\sum_{m \in  \mathcal{M}}\lambda_m \tau_m}{\zeta \sum_{n \in  \mathcal{N}_m}\mu_n t_n^{cmp}}+1\bigg),\label{a_}\\
    b^*=&\frac{\gamma \ln\Big(\frac{AY - \sqrt{4AY\sum_{m \in  \mathcal{M}}\lambda_m \tau_m +A^2Y^2}}{2\sum_{m \in  \mathcal{M}}\lambda_m \tau_m}+1\Big)}{-Y} \label{b_} .
\end{align}


\subsubsection{Solution of $T$ and $\boldsymbol{\tau}$}
The optimal solution of $a^*,b^*$ have been obtained under given accuracy $\epsilon$. According to problem (\ref{p1}), the optimal solution of $\boldsymbol{\tau}$ and $T$ can be given by:
\begin{align}
\tau_m^* &= \max_{n \in \mathcal{N}_m}\{a^* \cdot t_n^{cmp} + t_{n \to m}^{com}\},\\
T^*&=\max_{m \in \mathcal{M}}\{b^* \cdot \max_{n \in \mathcal{N}_m}\{a^* \cdot t_n^{cmp} + t_{n \to m}^{com}\} + t_{m \to c}^{com}\}.
\end{align}

\subsubsection{Lagrange multipliers update}
The Lagrange dual variables $\boldsymbol{\lambda},\boldsymbol{\mu},\boldsymbol{\beta},\boldsymbol{\nu}$ can be obtained by solving the Lagrange dual problem of problem (\ref{p2}), which can be expressed as follows:
\begin{subequations} \label{p_dual}
\begin{align}
    \max_{\boldsymbol{\lambda},\boldsymbol{\mu},\boldsymbol{\beta},\boldsymbol{\nu}}~ &g(\boldsymbol{\lambda},\boldsymbol{\mu},\boldsymbol{\beta},\boldsymbol{\nu}), \tag{\ref{p_dual}}\\
    \text{s.t. }
    &\boldsymbol{\lambda}\succeq 0, \boldsymbol{\mu}\succeq 0, \boldsymbol{\beta}\succeq 0, \boldsymbol{\nu}\succeq 0.
\end{align}
\end{subequations}
The Lagrange dual problem is a convex problem, which can be solved by subgradient projection method. The subgradients of $g(\boldsymbol{\lambda},\boldsymbol{\mu},\boldsymbol{\beta},\boldsymbol{\nu})$ can be given by:
\begin{align}\label{dual variables_1}
    \nabla \lambda_m&=b^*\tau_m^*  + t_{m \to c}^{com}-T^*,~\forall m \in \mathcal{M},\nonumber\\
    \nabla \mu_n&=a^* t_n^{cmp} + t_{n \to m}^{com}-\tau_m^*,~\forall n \in \mathcal{N}_m,\nonumber\\
    \nabla \beta_n&=f_n-f_n^{max},~\forall n \in \mathcal{N}_m,\nonumber\\
    \nabla \nu_n&=p_n-p_n^{max},~\forall n \in \mathcal{N}_m.
\end{align}
Having obtained the subgradients of $\lambda_m$, $\mu_n$, $\beta_n$, $\nu_n$, the Lagrange multipliers can be obtained by subgradients projection method in an iterative method as follows:
\begin{align}\label{dual variables_2}
    \lambda_m(t+1)&=\lambda_m(t)-\eta\nabla\lambda_m(t),\nonumber\\
    \mu_n(t+1)&=\mu_n(t)-\eta\nabla\mu_n(t),\nonumber\\
    \beta_n(t+1)&=\beta_n(t)-\eta\nabla\beta_n(t),\nonumber\\
    \nu_n(t+1)&=\nu_n(t)-\eta\nabla\nu_n(t),
\end{align}
where $\eta$ is the step size and $t$ denotes the number of iterations. The Lagrangian dual variables $\lambda_m$, $\mu_n$, $\beta_n$ and $\nu_n$ is updated according to the subgradients projection method in an iterative way in (\ref{dual variables_2}). After obtaining the Lagrangian dual variables, we substitute them into (\ref{partial derivatives}) to acquire the optimal value of $a$ and $b$. We then substitute the optimal solution of $a$ and $b$ into (\ref{dual variables_2}) to obtain the new Lagrangian dual variables in a cyclic iterative method.
Algorithm \ref{algorithm:lagrangian} explains the process. The complexity of Algorithm \ref{algorithm:lagrangian} is $\mathcal{O}(K\ln(\frac{1}{\epsilon_2}))$ with accuracy $\epsilon_2$ by using subgradients projection method.

\begin{algorithm}[t]
\caption{Optimal solution of $a$, $b$ and Lagrangian dual variables} \label{algorithm:lagrangian}
    \begin{algorithmic}[1]
        \STATE {Initialize Lagrangian dual variables $\boldsymbol{\lambda}(0)$, $\boldsymbol{\mu}(0)$, $\boldsymbol{\beta}(0)$ and $\boldsymbol{\nu}(0)$}
        \STATE{$t \gets 0$}
        \WHILE{accuracy $\epsilon_2$ is not achieved}
            \STATE{update $a$, $b$ in (\ref{a_}) and (\ref{b_}) using $\boldsymbol{\lambda}(t)$, $\boldsymbol{\mu}(t)$, $\boldsymbol{\beta}(t)$ and $\boldsymbol{\nu}(t)$.}
            \STATE{update Lagrangian dual variables $\boldsymbol{\lambda}(t+1)$, $\boldsymbol{\mu}(t+1)$, $\boldsymbol{\beta}(t+1)$ and $\boldsymbol{\nu}(t+1)$ using (\ref{dual variables_1}) and (\ref{dual variables_2}).}
            \STATE {$t \gets t+1$.}
        \ENDWHILE
    \end{algorithmic}
\end{algorithm}


\subsection{Solution of sub-problem \uppercase\expandafter{\romannumeral2}} 
In this part, we present the time-minimized UE-to-edge association scheme. With the optimal $a$, $b$, $\boldsymbol{f}$ and $\boldsymbol{p}$, the UE-to-edge association problem is equivalent to the following problem:
\begin{subequations}\label{p3}
\begin{align}
    \min_{\boldsymbol{\chi}}~ &\max_{n \in \mathcal{N}}\{a t_n^{cmp} + t_{n \to m}^{com}\},\tag{\ref{p3}}\\
    \text{s.t.~}
        &\chi_{n,m} \in \{0,1\},\forall n\in \mathcal{N},\forall m\in \mathcal{M}, \label{p3_1}\\
        &\sum_m \chi_{n,m} = 1,\forall n\in \mathcal{N},\label{p3_2}\\
        & \sum_n \chi_{n,m} B_n \leq \mathcal{B}, \forall m \in \mathcal{M}.\label{p3_3}
\end{align}
\end{subequations}
By introducing a slack variable $z$, problem (\ref{p3}) can be reformulated into an epigraph form as follows:
\begin{subequations}\label{p4}
\begin{align}
    \min_{\boldsymbol{\chi}}~ &z,\tag{\ref{p4}}\\
    \text{s.t.~}
        &a t_n^{cmp} + t_{n \to m}^{com} \leq z, \forall n \in \mathcal{N},\label{p4_1}\\
        &\chi_{n,m} \in \{0,1\},\forall n\in \mathcal{N},\forall m\in \mathcal{M}, \label{p4_2}\\
        &\sum_m \chi_{n,m} = 1,\forall n\in \mathcal{N},\label{p4_3}\\
        & \sum_n \chi_{n,m} B_n \leq \mathcal{B}, \forall m \in \mathcal{M}.\label{p4_4}
\end{align}
\end{subequations}
We should notice that when the problem (\ref{p4}) reaches the optimality, the maximum of the left-hand side of (\ref{p4_1}) is equal to $z$. Otherwise (i.e., the maximum of the left-hand side of (\ref{p4_1}) is less than $z$), we could decrease $z$ since it would bring a smaller objective value for (\ref{p4}).

Problem (\ref{p4}) is now a \textit{mixed integer linear programming} (MILP) problem, which can be solved by branch-and-bound algorithm. However, the computational complexity of branch-and-bound algorithm is exponential in general, and thereby cannot be implemented in practice. To solve the UE-to-edge association problem practically, we then propose a more efficient algorithm. 


In the proposed algorithm, we first identify the UEs with the largest SNR for each edge server successively, under the bandwidth constraint (\ref{p4_4}). Since each UE can only be associated with one edge server, we need to remove one edge server if two edge servers are associated with one UE. 
Specifically, let the set of UEs chosen by edge server $m_1$ and $m_2$ be $\mathcal{N}_{m_1}$ and $\mathcal{N}_{m_2}$. If a UE $n$ is in both $\mathcal{N}_{m_1}$ and $\mathcal{N}_{m_2}$, then the algorithm will compare the uplink channel SNR between UEs that are not in $\mathcal{N}_{m_1}$, $\mathcal{N}_{m_2}$ and edge server $m_1$, $m_2$, denoted by $\{ (n,m) \mid n \in \mathcal{N}\setminus(\mathcal{N}_{m_1}\cup\mathcal{N}_{m_2}),m \in \{m_1,m_2\}\}$. 
Then, the UE $n'$ and edge server $m'$ with largest uplink channel SNR $g_{n',m'}p_{n'}/N_0$ are chosen. If $m'=m_1$, then we remove UE $n$ from $m_1$ and associate $n'$ with $m_1$. Otherwise, we remove UE $n$ from $m_2$ and associate $n'$ with $m_2$. The above process proceeds until the last edge server finishes. The main procedures of the proposed algorithm are summarized in Algorithm \ref{algorithm:ue_to_edge}. In each round, a maximum of $\mathcal{B}/B_n$ comparisons are made. Hence, the complexity of proposed algorithm is $\mathcal{O}(m\mathcal{B}/B_n)$ in the worst case.

\begin{algorithm}[!]
\caption{Time-minimized UE-to-edge association algorithm} \label{algorithm:ue_to_edge}
    \begin{algorithmic}[1]
        \STATE {Initialize $\{\chi_{n,m}\}$ as empty and sort $g_{n,m}p_n/N_0$ by $m \in \mathcal{M}$.}
        \FOR{$i \in \{1,2,...,M\}$}
        \STATE{choose the $N_m$ UEs with largest $g_{n,m}p_n/N_0$, denote the set by $\mathcal{N}_{m_i}$, set $\chi_{n,i} = 1, \forall n \in \mathcal{N}_{m_i}$.}
        \WHILE{$\exists~ n$, $m_j$, $\chi_{n,m_i}==1$ and $\chi_{n,m_j}==1$ ($i > j$)}
        \STATE{$(n',m')=\mathop{\arg\max}\limits_{n \in \mathcal{N}\setminus(\mathcal{N}_{m_i}\cup\mathcal{N}_{m_j}),m \in \{m_i,m_j\}} g_{n,m}p_n/N_0$.}
        \STATE{set $\chi_{n,m'} = 0$.}
        \STATE{set $\chi_{n',m'} = 1$.}
        \ENDWHILE
        \ENDFOR
    \end{algorithmic}
\end{algorithm}

\section{Numerical Results} \label{section:numerical}
\begin{figure}[tb]
\centering
\centering\includegraphics[width=0.5\textwidth]{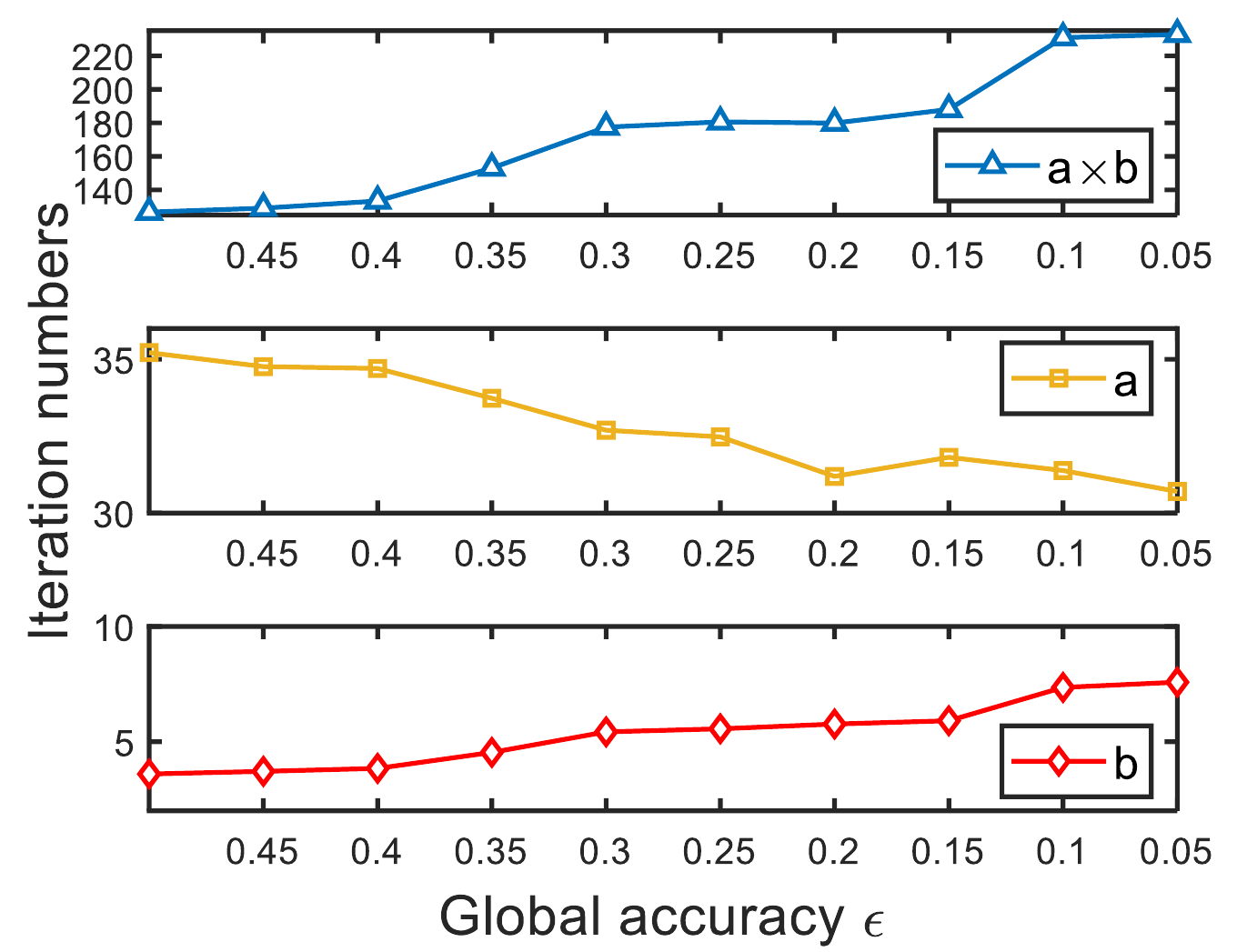}
\caption{Iterations under different global accuracy.}
\label{fig_acc}
\end{figure}

In this section, numerical experiments are conducted to verify the performance of our solutions. The advantages of hierarchical federated learning system over edge-based federated learning system and cloud-based federated learning system are investigated in~\cite{liu2020client}. Hence, in this paper, we focus on the iteration counts of local UEs and edge servers under different conditions, as well as UE-to-edge association.

\subsection{Experiment Settings}
For simulations, we consider a hierarchical federated learning system with multiple user equipments, edge servers and one cloud server. The user equipments are deployed in a square area of size 500 m $\times$ 500 m with the edge servers located in the center and all the edge servers are deployed in an area with the cloud server located in the center.  
For machine learning tasks, we consider a classification task using standard dataset MNIST. For the training model, we use LeNet. 
The constants $\gamma$, $\zeta$ and $\delta$ are set to random integers between 1 to 10. 
For simplicity, we use the free-space path loss model in~\cite{goldsmith2005wireless}.
Then with \texttt{wavelength} being the wavelength of the wireless signal and \texttt{distance} being the distance between UE $n$ and edge server $m$, it holds that $g_{n,m} = ( \frac{\texttt{wavelength}}{4 \pi\times \texttt{distance}} )^2$. 
We set the frequency at $28 \text{GHz}$ so that $\texttt{wavelength} = \frac{3 \times 10^8 }{28 \times 10^9} = \frac{3}{280} \text{m}$. 
Then $g_{n,m} = ( \frac{3/280}{4 \pi \times \texttt{distance}} )^2$. 
The maximum CPU frequency $f_n^{max}$ is 2 GHz, and maximum transmission power $p_n^{max}$ is 10 dBm for each device.
For the parameters in the optimization problem regarding Assumption \ref{Assumption}, $L$-smooth and $\beta$-strongly convex, we follow the experiment setting in~\cite{tran2019federated}.


\subsection{The optimal number of local computations and edge aggregations}
Firstly, we fix the number of UEs and edge servers. We deploy 1 cloud server, and 5 edge servers and each edge server is associated with 20 UEs. To achieve a given global accuracy $\epsilon$ within minimum time, the local iterations and edge iterations needed between two communication rounds are shown in Fig.~\ref{fig_acc}. As  $\epsilon$ decreases (i.e., higher machine learning model accuracy is required), $a$ decreases while $b$ increases, and the value of $a \times b$ (the number of local iterations in one cloud round) increases. It means that in order to obtain a more accurate global model within minimum time, edge servers need to run more edge iterations while UEs run fewer local iterations within one communication round. In the simulation experiment, we train LeNet on MNIST dataset with each edge server associated with 10 UEs. It can be observed from Fig. \ref{fig_training1} that under different required test accuracy, the optimal value of $a$ and $b$ differs. For example, in Fig. \ref{fig_training1}, if the required machine learning model accuracy is between 0.88 to 0.89, then $a=35, b=5$ is the optimal value. If the required machine learning model accuracy is beyond 0.92, then $a=30, b=7$ is the optimal value.

\begin{figure}[t]
\centering
\centering\includegraphics[width=0.5\textwidth]{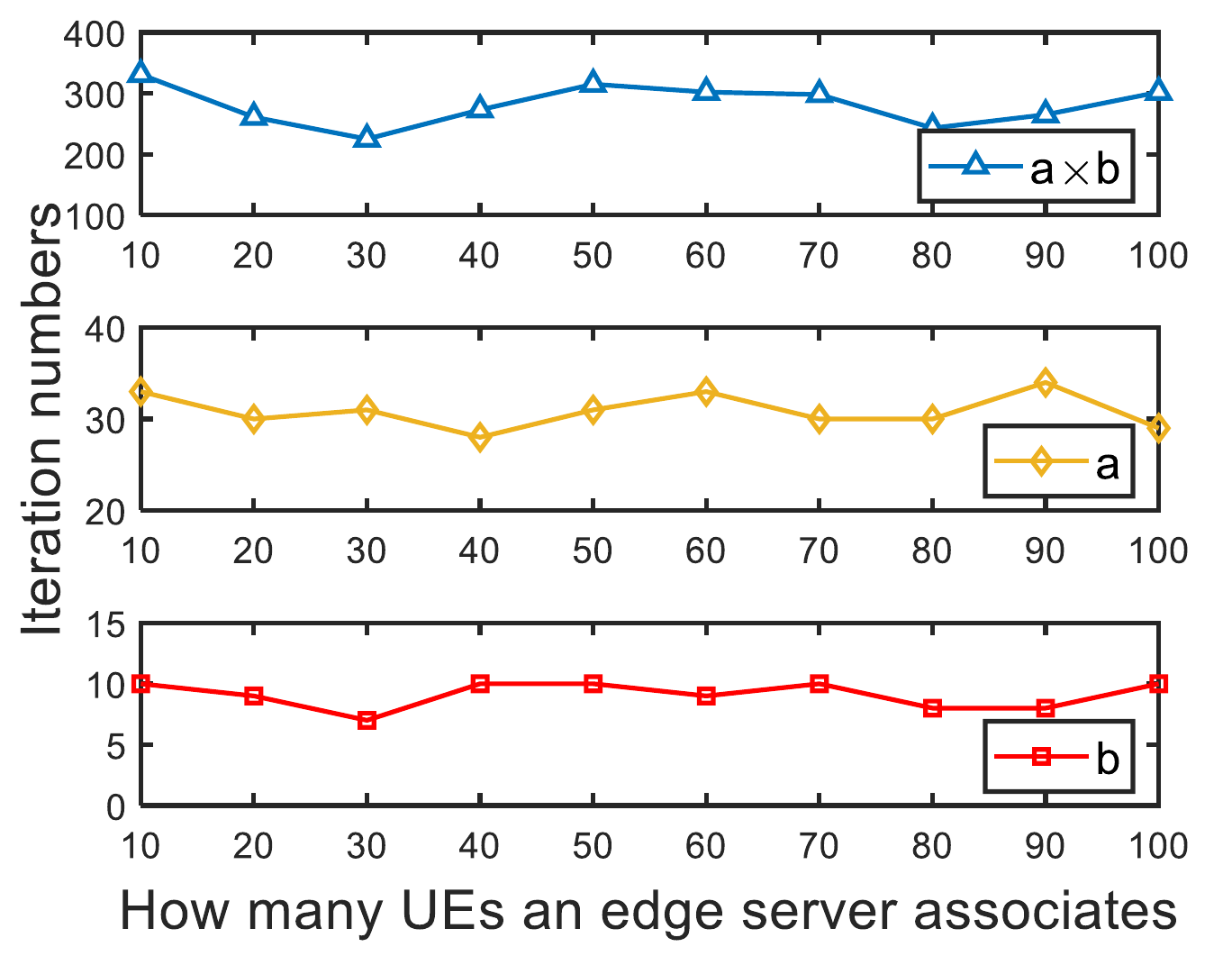}
\caption{Iterations under different numbers of UEs.}
\label{fig_ues}
\end{figure}

Next, we associate different numbers of UEs with each edge server from 10 UEs to 100 UEs. To obtain a fixed global accuracy in minimum time, the local iterations and edge iterations needed are shown in Fig. \ref{fig_ues}. As the number of UEs each edge server associates with increases, the number of local iterations and edge iterations exhibit no visible trend. That is because at the aggregation step, the weighted average scheme balances the variance among all the UEs. In Fig. \ref{fig_training2}, each edge server is associated with 20 UEs. It can be seen from Fig. \ref{fig_training2} that the optimal values of $a$ and $b$ are different when we require different machine learning model accuracy. Similar to the case when each edge server associates with 10 UEs, $a=35, b=5$ is the optimal value when the required model accuracy is between 0.88 to 0.89. If the required model accuracy is beyond 0.9, then $a=30, b=5$ is the optimal value. It also verifies the observation that the optimal value of local iteration counts and edge iteration counts have no correlation with the number of UEs edge server associates.

\begin{figure}[tb]
\centering\includegraphics[width=0.5\textwidth]{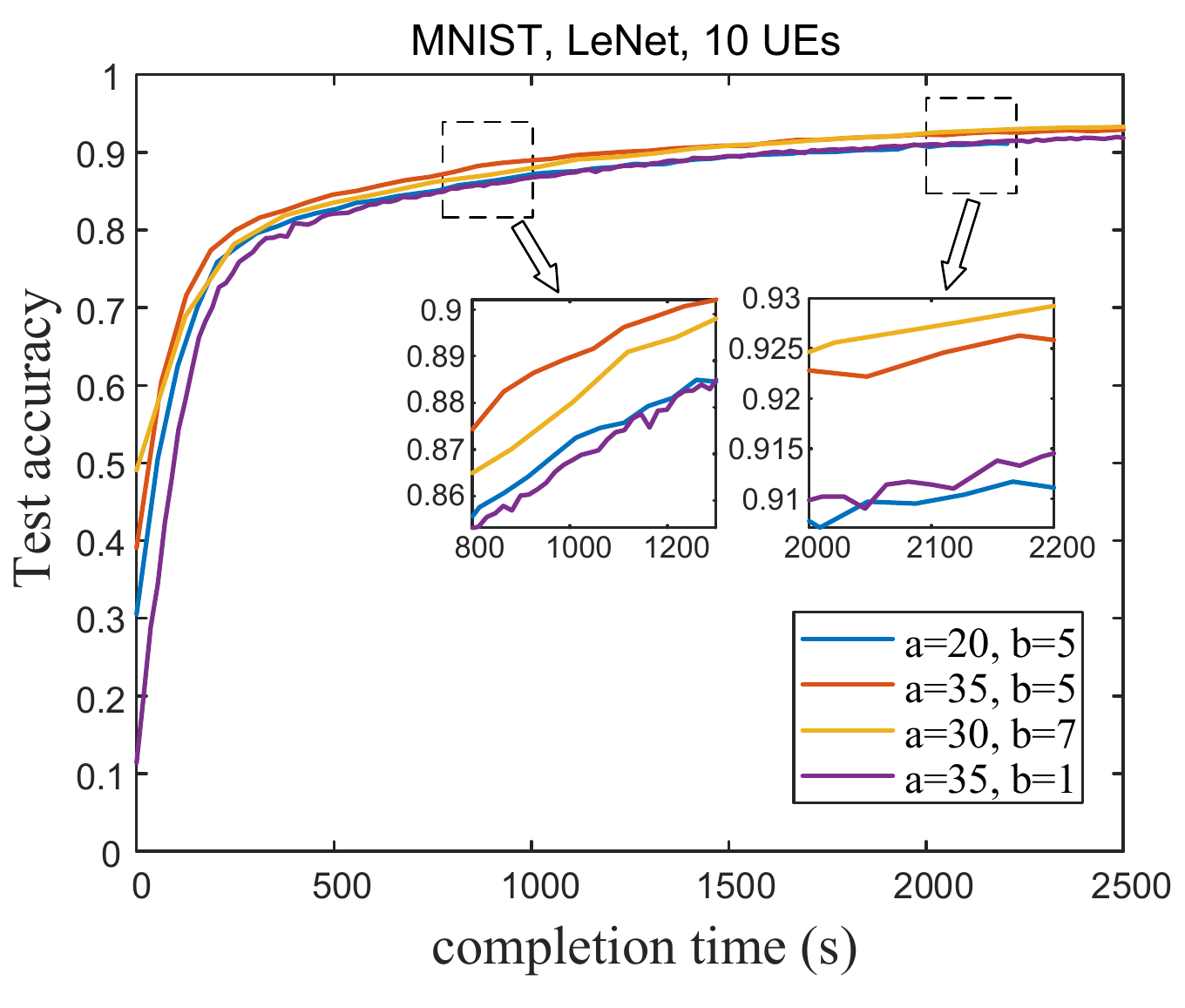}  
\caption{Test accuracy on MNIST w.r.t. the completion time when each edge server is associated with 10 UEs.}
\label{fig_training1}
\end{figure}

\begin{figure}[tb]
\centering{\includegraphics[width=0.5\textwidth]{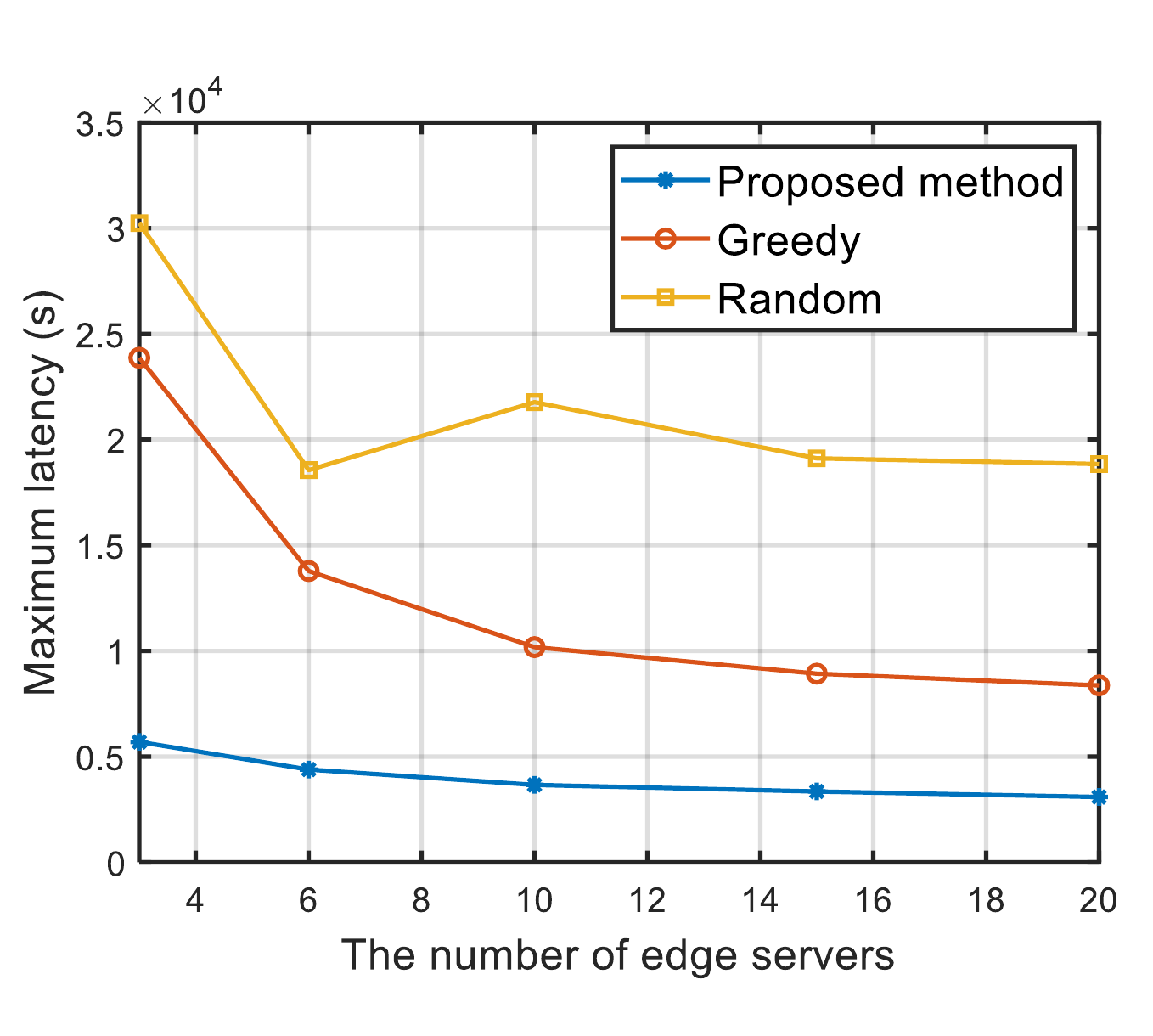}}
\caption{The maximum latency of 100 UEs under different numbers of edge servers.}
\label{fig_asso}
\end{figure}

\begin{figure}[t]
\centering\includegraphics[width=0.5\textwidth]{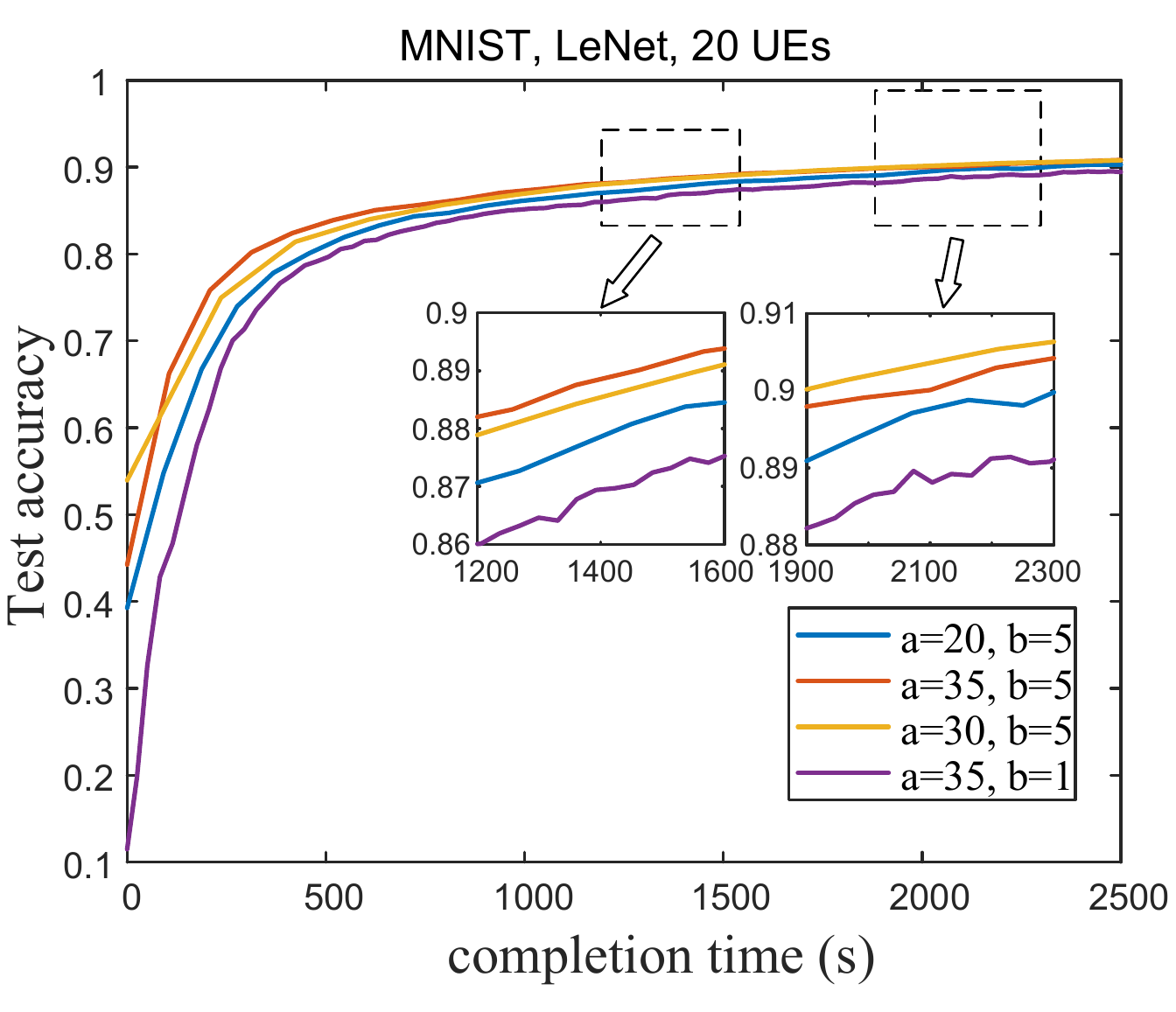}  
\caption{Test accuracy on MNIST w.r.t. the completion time when each edge server is associated with 20 UEs.}
\label{fig_training2}
\end{figure}

\subsection{The optimal UE-to-edge association}
In this part, we test three different UE-to-edge association strategies under the global accuracy requirement $\epsilon=0.25$. They are the proposed method, the greedy algorithm and the random UE-to-edge association strategy. 
\begin{itemize}
    \item \textbf{Greedy algorithm :} The greedy algorithm chooses the UEs available with maximum SNR under the bandwidth constraint for each edge server.
    \item \textbf{Random UE-to-edge association:} The random UE-to-edge association assigns UEs to all the edge servers randomly, under bandwidth constraint. 
\end{itemize}
Fig.~\ref{fig_asso} illustrates the system's maximum latency under different numbers of edge servers. The number of total UEs is 100. It can be found that the proposed method always achieves the lowest latency under different numbers of edge servers. The greedy algorithm always chooses the UEs with the maximum SNR available, thus outperforming the random association.
It is noticed that when the number of edge servers is smaller, the latency is higher. That is because when the number of edge servers is small, UEs have few choices. More UEs have to associate with edge servers whose SNR is high.

\section{Conclusion} \label{section:conclusion}
In this paper, we have investigated the problem of latency minimization in the setting of 3-layer hierarchical federated learning framework. 
In particular, we formulated a joint learning and communication problem, where we optimized local iteration counts and edge server iteration counts. To solve the problem, we studied the convexity of the problem.
Then we proposed an iterative algorithm to obtain the optimal solution for local counts and edge counts.
Besides, we proposed a UE-to-edge association strategy which aims to minimize the maximum latency of the system.
Simulation results show the performance of our solutions, where the global model converges faster under optimized number of local iterations and edge aggregations. The overall FL time is minimized with the proposed UE-to-edge association strategy.

\section*{Acknowledgement}

This research is supported in part by Nanyang Technological University Startup Grant; in part by the Singapore Ministry of Education Academic Research Fund under Grant Tier 1 RG97/20, Grant Tier 1 RG24/20 and Grant Tier 2 MOE2019-T2-1-176.



\begin{thebibliography}{10}
\providecommand{\url}[1]{#1}
\csname url@samestyle\endcsname
\providecommand{\newblock}{\relax}
\providecommand{\bibinfo}[2]{#2}
\providecommand{\BIBentrySTDinterwordspacing}{\spaceskip=0pt\relax}
\providecommand{\BIBentryALTinterwordstretchfactor}{4}
\providecommand{\BIBentryALTinterwordspacing}{\spaceskip=\fontdimen2\font plus
\BIBentryALTinterwordstretchfactor\fontdimen3\font minus
  \fontdimen4\font\relax}
\providecommand{\BIBforeignlanguage}[2]{{%
\expandafter\ifx\csname l@#1\endcsname\relax
\typeout{** WARNING: IEEEtran.bst: No hyphenation pattern has been}%
\typeout{** loaded for the language `#1'. Using the pattern for}%
\typeout{** the default language instead.}%
\else
\language=\csname l@#1\endcsname
\fi
#2}}
\providecommand{\BIBdecl}{\relax}
\BIBdecl

\bibitem{bonawitz2019towards}
K.~Bonawitz, H.~Eichner, W.~Grieskamp, D.~Huba, A.~Ingerman, V.~Ivanov,
  C.~Kiddon, J.~Kone{\v{c}}n{\`y}, S.~Mazzocchi, H.~B. McMahan \emph{et~al.},
  ``Towards federated learning at scale: System design,'' \emph{arXiv preprint
  arXiv:1902.01046}, 2019.

\bibitem{haddadpour2019convergence}
F.~Haddadpour and M.~Mahdavi, ``On the convergence of local descent methods in
  federated learning,'' \emph{arXiv preprint arXiv:1910.14425}, 2019.

\bibitem{zhao2018federated}
Y.~Zhao, M.~Li, L.~Lai, N.~Suda, D.~Civin, and V.~Chandra, ``Federated learning
  with non-iid data,'' \emph{arXiv preprint arXiv:1806.00582}, 2018.

\bibitem{amiri2020machine}
M.~M. Amiri and D.~G{\"u}nd{\"u}z, ``Machine learning at the wireless edge:
  Distributed stochastic gradient descent over-the-air,'' \emph{IEEE
  Transactions on Signal Processing}, vol.~68, pp. 2155--2169, 2020.

\bibitem{ahn2019wireless}
J.-H. Ahn, O.~Simeone, and J.~Kang, ``Wireless federated distillation for
  distributed edge learning with heterogeneous data,'' in \emph{2019 IEEE 30th
  Annual International Symposium on Personal, Indoor and Mobile Radio
  Communications (PIMRC)}.\hskip 1em plus 0.5em minus 0.4em\relax IEEE, 2019,
  pp. 1--6.

\bibitem{chen2019artificial}
M.~Chen, U.~Challita, W.~Saad, C.~Yin, and M.~Debbah, ``Artificial neural
  networks-based machine learning for wireless networks: A tutorial,''
  \emph{IEEE Communications Surveys \& Tutorials}, vol.~21, no.~4, pp.
  3039--3071, 2019.

\bibitem{tran2019federated}
N.~H. Tran, W.~Bao, A.~Zomaya, N.~M. NH, and C.~S. Hong, ``Federated learning
  over wireless networks: Optimization model design and analysis,'' in
  \emph{IEEE INFOCOM 2019-IEEE Conference on Computer Communications}.\hskip
  1em plus 0.5em minus 0.4em\relax IEEE, 2019, pp. 1387--1395.

\bibitem{chen2020convergence}
M.~Chen, H.~V. Poor, W.~Saad, and S.~Cui, ``Convergence time optimization for
  federated learning over wireless networks,'' \emph{IEEE Transactions on
  Wireless Communications}, 2020.

\bibitem{shi2020device}
W.~Shi, S.~Zhou, and Z.~Niu, ``Device scheduling with fast convergence for
  wireless federated learning,'' in \emph{ICC 2020-2020 IEEE International
  Conference on Communications (ICC)}.\hskip 1em plus 0.5em minus 0.4em\relax
  IEEE, 2020, pp. 1--6.

\bibitem{chen2020joint}
M.~Chen, Z.~Yang, W.~Saad, C.~Yin, H.~V. Poor, and S.~Cui, ``A joint learning
  and communications framework for federated learning over wireless networks,''
  \emph{IEEE Transactions on Wireless Communications}, 2020.

\bibitem{wang2019adaptive}
S.~Wang, T.~Tuor, T.~Salonidis, K.~K. Leung, C.~Makaya, T.~He, and K.~Chan,
  ``Adaptive federated learning in resource constrained edge computing
  systems,'' \emph{IEEE Journal on Selected Areas in Communications}, vol.~37,
  no.~6, pp. 1205--1221, 2019.

\bibitem{yang2020delay}
Z.~Yang, M.~Chen, W.~Saad, C.~S. Hong, M.~Shikh-Bahaei, H.~V. Poor, and S.~Cui,
  ``Delay minimization for federated learning over wireless communication
  networks,'' \emph{arXiv preprint arXiv:2007.03462}, 2020.

\bibitem{luo2020hfel}
S.~Luo, X.~Chen, Q.~Wu, Z.~Zhou, and S.~Yu, ``Hfel: Joint edge association and
  resource allocation for cost-efficient hierarchical federated edge
  learning,'' \emph{IEEE Transactions on Wireless Communications}, vol.~19,
  no.~10, pp. 6535--6548, 2020.

\bibitem{wang2020local}
J.~Wang, S.~Wang, R.-R. Chen, and M.~Ji, ``Local averaging helps: Hierarchical
  federated learning and convergence analysis,'' \emph{arXiv preprint
  arXiv:2010.12998}, 2020.

\bibitem{liu2020client}
L.~Liu, J.~Zhang, S.~Song, and K.~B. Letaief, ``Client-edge-cloud hierarchical
  federated learning,'' in \emph{ICC 2020-2020 IEEE International Conference on
  Communications (ICC)}.\hskip 1em plus 0.5em minus 0.4em\relax IEEE, 2020, pp.
  1--6.

\bibitem{lim2021dynamic}
W.~Y.~B. Lim, J.~S. Ng, Z.~Xiong, D.~Niyato, C.~Miao, and D.~I. Kim, ``Dynamic
  edge association and resource allocation in self-organizing hierarchical
  federated learning networks,'' \emph{IEEE Journal on Selected Areas in
  Communications}, vol.~39, no.~12, pp. 3640--3653, 2021.

\bibitem{mhaisen2021optimal}
N.~Mhaisen, A.~A. Abdellatif, A.~Mohamed, A.~Erbad, and M.~Guizani, ``Optimal
  user-edge assignment in hierarchical federated learning based on statistical
  properties and network topology constraints,'' \emph{IEEE Transactions on
  Network Science and Engineering}, vol.~9, no.~1, pp. 55--66, 2021.

\bibitem{abdellatif2022communication}
A.~A. Abdellatif, N.~Mhaisen, A.~Mohamed, A.~Erbad, M.~Guizani, Z.~Dawy, and
  W.~Nasreddine, ``Communication-efficient hierarchical federated learning for
  iot heterogeneous systems with imbalanced data,'' \emph{Future Generation
  Computer Systems}, vol. 128, pp. 406--419, 2022.

\bibitem{wang2021quantized}
Y.~Wang, Y.~Xu, Q.~Shi, and T.-H. Chang, ``Quantized federated learning under
  transmission delay and outage constraints,'' \emph{IEEE Journal on Selected
  Areas in Communications}, vol.~40, no.~1, pp. 323--341, 2021.

\bibitem{chen2021communication}
M.~Chen, N.~Shlezinger, H.~V. Poor, Y.~C. Eldar, and S.~Cui,
  ``Communication-efficient federated learning,'' \emph{Proceedings of the
  National Academy of Sciences}, vol. 118, no.~17, p. e2024789118, 2021.

\bibitem{yang2020energy}
Z.~Yang, M.~Chen, W.~Saad, C.~S. Hong, and M.~Shikh-Bahaei, ``Energy efficient
  federated learning over wireless communication networks,'' \emph{IEEE
  Transactions on Wireless Communications}, 2020.

\bibitem{shamir2014communication}
O.~Shamir, N.~Srebro, and T.~Zhang, ``Communication-efficient distributed
  optimization using an approximate newton-type method,'' in
  \emph{International conference on machine learning}.\hskip 1em plus 0.5em
  minus 0.4em\relax PMLR, 2014, pp. 1000--1008.

\bibitem{schrijver1998theory}
A.~Schrijver, \emph{Theory of linear and integer programming}.\hskip 1em plus
  0.5em minus 0.4em\relax John Wiley \& Sons, 1998.

\bibitem{goldsmith2005wireless}
A.~Goldsmith, \emph{Wireless communications}.\hskip 1em plus 0.5em minus
  0.4em\relax Cambridge university press, 2005.

\end{thebibliography}
\end{document}